\documentclass{article}

\PassOptionsToPackage{numbers,sort&compress}{natbib}
\usepackage[preprint]{neurips_2026}
\workshoptitle{Evaluation of Interactive Agents}

\usepackage[utf8]{inputenc}
\usepackage[T1]{fontenc}
\usepackage{hyperref}
\hypersetup{hidelinks,hypertexnames=false}
\usepackage{url}
\usepackage{booktabs}
\usepackage{graphicx}
\usepackage{amsmath,amssymb,amsfonts,amsthm}
\usepackage{microtype}
\usepackage{xcolor}
\usepackage{float}
\usepackage{algorithm}
\usepackage[noend]{algpseudocode}

\newtheorem{theorem}{Theorem}
\newtheorem{lemma}{Lemma}

\newcommand{\E}{\mathbb{E}}
\newcommand{\Pp}{\mathbb{P}}
\newcommand{\Acal}{\mathcal{A}}
\newcommand{\Dcal}{\mathcal{D}}
\newcommand{\Ical}{\mathcal{I}}
\newcommand{\Jcal}{\mathcal{J}}

\newcommand{\KMS}{\operatorname{KMS}}
\newcommand{\Wass}{\operatorname{W}}
\newcommand{\SASST}{\textsc{SASST}}
\newcommand{\ESS}{\operatorname{ESS}}

\algrenewcommand\algorithmicrequire{\textbf{Input:}}
\algrenewcommand\algorithmicensure{\textbf{Output:}}
\title{Selection-Aware Stress Testing for Interactive Agents}
\author{%
  Yang Xu\thanks{Equal contribution.} \\
  Purdue University\\
  \texttt{xu1720@purdue.edu} \\
  \And
  Chenang Li\footnotemark[1] \\
  University of California, Irvine\\
  \texttt{chenangl@uci.edu} \\
  \And
  Jiefu Zhang \\
  Purdue University\\
  \texttt{zhan4018@purdue.edu} \\
  \And
  Haixiang Sun\\
  Purdue University\\
  \texttt{sun1321@purdue.edu} \\
  \And
  Zhou Li\\
  University of California, Irvine\\
  \texttt{zhou.li@uci.edu} \\
  \And
  Vaneet Aggarwal \\
  Purdue University \\
  \texttt{vaneet@purdue.edu}
}

\begin{document}
\maketitle

\begin{abstract}
Agent evaluations often use one benchmark to choose a workflow and then search for task types where its advantage weakens, so both conclusions are
selected from the same data. We introduce Selection-Aware Semantic Stress Testing (\SASST{}), which learns a task reweighting from pre-execution features on discovery tasks and evaluates the same paired comparison on separate confirmation tasks. The protocol checks support and stability, uses joint bounds for all planned claims, and can return no claim. We prove conditional asymptotic validity under stated cluster assumptions. A forty-cluster audit finds Gaussian undercoverage and conservative Bonferroni $t$ bounds. In one 480-episode $\tau$-bench study, a 3.75 point discovery gain vanished on confirmation. A second-model study likewise confirmed neither a workflow benefit nor a stable stress rule.
\end{abstract}

\section{Introduction}

Interactive-agent evaluations increasingly inform workflow selection \citep{anthropic2026agentevals,nist2026agentstandards}. A team may compare workflows that interleave reasoning and acting, add reflection or search-based planning, or compile retrieved skills \citep{yao2023react,shinn2023reflexion,zhou2024language,meng2026skillrae},
then select the strongest on a heterogeneous benchmark. However, that winner may lose
its advantage on a coherent mixture of subset tasks. If confirmed, such a reversal could support reconsidering the workflow choice for that task regime. Teams may therefore search many task types for where the selected workflow
loses its overall advantage. Because the subgroup is chosen after inspecting these same outcomes, an extreme reversal may reflect sampling noise rather than a pattern that persists on new tasks, making independent confirmation necessary before changing the deployment decision. In existing works, \citep{xu2026siren} protects workflow comparison,
\citep{eyuboglu2022domino,sohoni2020george,deon2022spotlight,yang2023subpopbench} find coherent error groups, \citep{cherian2024fairness} protects fixed group families, and \citep{repantis2026janus} confirms frozen descriptors for a fixed model. These methods do not by themselves protect the full process of choosing an average winner and then searching for task types where its advantage weakens or reverses. We propose Selection-Aware Semantic Stress Testing (SASST). It keeps repeated
runs of each task together and divides tasks into two sets: a discovery set for choosing a reweighting rule and a confirmation set for evaluating that frozen rule on tasks not used to choose it. Discovery-set workflow scores are used to choose among candidate rules based only on pre-execution attributes, such as goal length, required tools, or permission flags. Before any results from the confirmation set are inspected, the chosen rule is fixed. It then assigns weights to those tasks, and the same workflows are compared again. \SASST{} therefore protects two data-dependent choices: the workflow winner and the rule that most weakens its advantage. It rejects rules supported by too few tasks or unstable under small changes to the discovery data, and may
return no claim.

\textbf{Contributions.} \SASST{} (i) asks whether a workflow that wins on
average still wins when a coherent task type receives more emphasis, with all
workflows evaluated on the same task mix; (ii) learns from discovery results a
reusable rule based on task attributes known before execution, rejecting rules
that rely on too few tasks or exceed a declared shift; (iii) freezes the rule
and retests the same comparison on separate confirmation tasks, with
uncertainty bounds for all planned claims; and (iv) returns confirmed, not
confirmed, failed discovery, failed feasibility, or abstained, distinguishing
supported, inconclusive, infeasible, and unstable findings. Theory, controlled studies, and agent studies evaluate the procedure.

\section{Method}

\paragraph{Tasks and the fixed comparison.}
Let $W\sim P$ denote one benchmark task before execution. It contains the goal, initial state, available tools, permissions, and other fixed task information. Scores produced by running workflows under different simulator seeds are observed outcomes, not being a part of $W$. All outcomes from the same task remain together as one task-level cluster. Before choosing a task-weighting rule, we fix the comparison plan denoted as $\Acal$ for both the discovery and the confirmation. Given the observed workflow scores and a set of task weights, $\Acal$ returns one performance estimate for each workflow. It also specifies how repeated scores are combined and, when applicable, how workflows are selected or weighted. Let $\boldsymbol{\theta}(Q;\Acal)$ be these estimates when $Q$ provides the task weights, and let $c$ specify the comparison target of interest. We define $T_c(Q;\Acal)=c^\top\boldsymbol{\theta}(Q;\Acal)$. For example, when $c$ aims at comparing Plan + Verifier with ReAct \citep{yao2023react}, $T_c(Q;\Acal)$ is their difference in weighted average success under the same task weights. Throughout the setting, \SASST{} aims at changing task emphasis, not the workflows or the comparison plan. The stress target defined next asks whether this same comparison weakens when coherent task types receive more weight.

\paragraph{Stress target.}
A rule $g$ uses only pre-execution task features to assign a bounded
nonnegative score $s_g(W)$. Larger scores give a task more emphasis.
With $P s_g=\E_P[s_g(W)]$, define
\begin{equation}
  \omega_g(W)=\frac{s_g(W)}{P s_g},
  \qquad dP^g=\omega_g\,dP,
  \qquad \Delta(g)=T_c(P^g;\Acal).
  \label{eq:target}
\end{equation}
Here $P^g$ is the benchmark distribution reweighted by $g$, and $\Delta(g)$ is the workflow contrast under that emphasis. It covers represented task types and allowed shifts, not absent tasks or arbitrary deployment traffic. We next explain how $g$ is chosen on discovery data and evaluated independently.

\paragraph{Discovery and confirmation.}
On discovery clusters $\Dcal$, the equal-weight comparison gives
$\widehat\Delta_{\Dcal}$. The contribution $\widehat\psi_i$
approximates how changing task $i$'s weight changes the final
comparison, including any workflow-selection step inside $\Acal$.
With
$w_i(g;\Dcal)=s_g(W_i)/\sum_{\ell\in\Dcal}s_g(W_\ell)$,
rank candidate rules by
\begin{equation}
  \widehat\Delta^{\mathrm{loc}}_{\Dcal}(g)
  =\widehat\Delta_{\Dcal}
  +\sum_{i\in\Dcal}
  \left\{w_i(g;\Dcal)-\frac{1}{|\Dcal|}\right\}\widehat\psi_i.
  \label{eq:local}
\end{equation}
Lower values predict a weaker contrast. Discovery outcomes choose among candidate functions, but every candidate uses only pre-execution features. The selected function is frozen before confirmation. Equation~\eqref{eq:local} ranks rules only during discovery. Confirmation later recomputes the exact contrast.

Density and effective sample size ($\ESS$) checks prevent a few tasks from dominating. A radius limits the distance between the original and reweighted distributions of fixed task features. We use a finite-witness Kernel Max-Sliced Wasserstein (KMS) approximation \citep{wang2025kms}, with maximum mean discrepancy (MMD) and linear max-sliced alternatives \citep{gretton2012kernel}. Resampling the discovery clusters checks rule stability. On confirmation clusters $\Ical$, \SASST{} applies the frozen rule, repeats the support checks, recomputes the exact paired comparison, and forms a joint uncertainty bound over all planned claims. Each planned claim fixes its contrast, direction, candidate rule family, allowed radius, thresholds, and inference rule. Algorithm~\ref{alg:sasst} combines these choices, support checks, and
confirmation into one decision.

\begin{algorithm}[H]
\caption{\SASST{} discovery and confirmation.}
\label{alg:sasst}
\footnotesize
\begin{algorithmic}[1]
\Require Fixed comparison plan $\Acal$, independent task clusters,
pre-execution features, planned claims $\Jcal_0$.
\State Split whole clusters into discovery $\Dcal$ and confirmation
$\Ical$. Freeze workflows, candidate rule families, thresholds, and
the inference rule.
\State On $\Dcal$, select for each claim the feasible rule minimizing
Equation~\eqref{eq:local}. If none is feasible, record
\emph{failed discovery}.
\State Freeze the rule and estimate its stability using only $\Dcal$.
Apply it unchanged to $\Ical$. If support fails, record
\emph{failed feasibility}.
\State For every support-feasible rule, compute the exact weighted
comparison on $\Ical$ and its joint uncertainty bound.
\State Record \emph{confirmed} or \emph{not confirmed} when stability
passes. Otherwise report the estimate diagnostically and record
\emph{abstained}.
\Ensure A complete status and diagnostic report for every planned
claim.
\end{algorithmic}
\end{algorithm}
The algorithm above defines the operational report. The theorem below states when its confirmation bounds are jointly valid.

\begin{theorem}[Validity after discovery]
\label{thm:validity}
Condition on the discovery data $\mathcal F_{\Dcal}$. For a finite
predeclared family, suppose the confirmation cluster vectors
$H_k=(H_{k,m}:m\in\Jcal_0)$ are conditionally independent and
identically distributed. Let
$\widehat\Delta_m=F_m(N^{-1}\sum_{k=1}^N H_{k,m})$,
$\mu_m=\E(H_{1,m}\mid\mathcal F_{\Dcal})$, and
$\Delta_m=F_m(\mu_m)$, where $H_{k,m}$ contains every weighted
numerator and denominator, every selector input, and both sides of the
paired contrast. Assume bounded density ratios, denominators bounded
away from zero, continuously differentiable $F_m$ near $\mu_m$, finite
cluster second moments, a finite-dimensional Lindeberg condition,
conditional covariance converging to $\Sigma$, and positive limiting
variances for reported coordinates. Then
$\sqrt N(\widehat\Delta_m-\Delta_m:m\in\Jcal_0)$ is asymptotically
Gaussian, and the cluster Gaussian multiplier consistently
approximates its joint law. Pointwise intervals and signed
simultaneous bands have asymptotic conditional coverage, including for
a coordinate highlighted after inspecting the protected family.
\end{theorem}

In plain language, once discovery fixes the rules, confirmation
supports joint uncertainty statements for all planned claims.
Appendix~\ref{app:proofs} gives the proof
\citep{chernozhukov2013gaussian}. The theorem assumes identically
distributed confirmation clusters, so the stratified designs in
Studies F and G require an extension not supplied here. With hard
workflow selection, the protected objects are the candidate
contrasts, not the argmax. Section~3 audits the forty-cluster regime.
\section{Validation}

We test whether confirmation preserves coverage after rule discovery, what fails with forty clusters, and whether a supported signal remains detectable. For $M\ge1000$, pointwise coverage is 94.75--95.30\% and six-slot simultaneous coverage is 94.95--95.50\%. At $M=1000$, same-data worst-tercile coverage is 91.6\%, versus 93.6\% with discovery and confirmation separated
($p=0.032$).

At $N=40$, an initial audit found 92.2--92.8\% Gaussian coverage; a later fresh-seed diagnostic found 90.55--93.50\% Gaussian and 96.40--97.55\%
Bonferroni-$t$ coverage across four null settings. With 200 clusters and adequate stressed $\ESS$, a planted-fragility control selects the target rule in 92.1\% of trials, confirms it conditionally in 70.1\%, detects it end-to-end in 64.6\%, and has 1.7--1.8\% null familywise error.

\begin{table}[H]
\centering
\small
\setlength{\tabcolsep}{3.4pt}
\caption{Inferential checks and the small-cluster safeguard.}
\label{tab:validation}
\begin{tabular}{lcc}
\toprule
Setting & Result & Reading \\
\midrule
Pointwise, $M\ge1000$ & 94.75--95.30\% & near 95\% \\
Six-slot simultaneous & 94.95 / 95.50\% & near 95\% \\
Ten-slot Gaussian, $N=40$ & 92.2--92.8\% & undercoverage \\
Ten-slot Bonferroni $t$, $N=40$ & 96.4--97.55\% & conservative \\
Positive control, $N=200$ & 64.6\% end-to-end & signal found \\
\bottomrule
\end{tabular}
\end{table}

Additional audits test the rule-search safeguards: no geometry dominates, a declared radius can exclude a planted rule, support checks reject concentrated
rules, and a detectable real-feature effect is reported as abstained when its description is unstable.

\section{Agent studies}

We apply the full protocol to two models. ReAct \citep{yao2023react} is the baseline tool-use loop; Plan + Verifier adds planning and model-based verification; Confirmation Gate requests approval before declared irreversible actions. For each model, two slots test average success against ReAct, six test the same comparisons under three allowed task shifts, and two compare the enhanced workflows. Each study uses the same 80 $\tau$-bench tasks \citep{yao2024taubench}, two simulator
seeds, and three workflows (480 episodes). Task ID is the independent cluster, whole tasks are split 40/40 for discovery and confirmation, and rules below the
predeclared 0.60 stability threshold are abstained.

\begin{table}[H]
\centering
\small
\setlength{\tabcolsep}{3.2pt}
\caption{Across-model results. Confirmation is the discovery-selected workflow minus ReAct; counts are confirmed / not confirmed / abstained / failed feasibility.}
\label{tab:agents}
\begin{tabular}{lccc}
\toprule
Model & Discovery result & Confirmation & Final counts \\
\midrule
Qwen3-8B & Plan+Verifier ($+3.75$ pp) & $0.0$ pp & 0 / 2 / 8 / 0 \\
Qwen2.5-7B & Confirmation Gate selected & $+2.5$ pp & 0 / 2 / 6 / 2 \\
\bottomrule
\end{tabular}
\end{table}

The Qwen3 discovery advantage disappeared on confirmation. Qwen2.5 had lower overall success, selected a different workflow, and produced different stress rules. Maximum rule stability was only 0.37 and 0.27, respectively, so neither model yielded a reusable stress description or confirmed workflow benefit.

These studies do not establish equivalence or safety. Forty confirmation clusters and stressed $\ESS$ near 20 provide little power for effects of practical size, many tasks fail under every workflow, and the Qwen3 kernel bandwidth made most task embeddings nearly orthogonal. The second study reuses the same task pool, so it is a model replication rather than an independent benchmark replication.

\section{Decision interpretation and scope}

\begin{table}[H]
\centering
\small
\setlength{\tabcolsep}{4pt}
\caption{Decision status returned for every predeclared slot.}
\label{tab:status}
\begin{tabular}{lp{0.72\linewidth}}
\toprule
Status & Meaning \\
\midrule
Confirmed & Support and stability pass, and the joint uncertainty bound supports the planned claim. \\
Not confirmed & Checks pass, but the bound does not support the claim; this is not equivalence. \\
Failed discovery & No candidate rule satisfies the discovery constraints. \\ Failed feasibility & The selected task weighting lacks confirmation support
or violates a predeclared check. \\ Abstained & The effect can be estimated, but the rule is too unstable to reuse. \\
\bottomrule
\end{tabular}
\end{table}

The positive control shows that the protocol can confirm a planted fragility when the rule is eligible and confirmation has adequate support. The Qwen3
study shows the opposite case: a $+3.75$ pp discovery gain and an apparently weak segment could have motivated a verifier, but the average gain vanishes on
confirmation and the segment description is unstable. The protected report supports neither claim and records what is missing rather than searching again
on the same confirmation-task results.

\paragraph{Reported evidence.}
For every planned claim, \SASST{} reports the frozen rule, estimated comparison, uncertainty bound, support, stability, geometry diagnostics, and
final status. These fields explain a missing claim: low support calls for more independent task clusters; instability calls for better features or rule geometry; and a non-confirmed bound calls for more power or a redesigned intervention.

\paragraph{Limitations.}
The theorem requires conditionally i.i.d. confirmation clusters, bounded task weights, stable denominators, honest discovery, and smooth comparison functions. The Bonferroni $t$ safeguard is supported by four audited null settings rather than a universal finite-sample theorem. Empirically, the studies use one $\tau$-bench task pool, low-success open models, the same model as actor and user simulator, and forty confirmation tasks; the second model is not an independent benchmark replication. The next protected study needs fresh task clusters, score-blind pilot calibration of geometry and radius, an outcome aligned with the intended policy or state transition, an independent checker where relevant, and a power plan based on stressed $\ESS$; more seeds on the same tasks do not replace independent clusters.

\paragraph{Operational scope.}
\SASST{} supports decisions only for the benchmark and allowed task shifts; it is not a deployment-security certificate. It distinguishes supported, underpowered, out-of-scope, and unstable claims before a workflow or control is released.

\clearpage
\bibliographystyle{plainnat}
\bibliography{main}

\clearpage
\appendix
\section{Related work and industry context}
\label{app:related}

Agent benchmarks now test tool use, persistent state, web work, and general task completion. Examples include $\tau$-bench, ToolSandbox, WorkArena, AgentBench, GAIA, and TaskBench \citep{yao2024taubench,lu2024toolsandbox,drouin2024workarena,liu2024agentbench,mialon2024gaia,shen2024taskbench}. The benchmarks make workflow evaluation concrete, but they do not provide after discovery confidence bounds for a task mixture selected after workflow search.

Several methods discover weak groups or protect group comparisons. GEORGE, Domino, Spotlight, and SubpopBench study hidden groups, systematic errors, and subpopulation shift \citep{sohoni2020george,eyuboglu2022domino,deon2022spotlight,yang2023subpopbench}. Fairness auditing gives simultaneous inference over rich group families \citep{cherian2024fairness}, while Janus uses decoys and held out replication for a frozen descriptor list and a fixed model \citep{repantis2026janus}. Distributionally robust optimization trains models for worst group or nearby distribution performance \citep{duchi2021uniform,sagawa2020groupdro}. SASST instead audits a frozen adaptive report. It discovers a common task reweighting, recomputes the complete paired report on independent clusters, and allows non-confirmation or abstention.

KMS and MMD provide two choices of geometry for comparing feature distributions \citep{wang2025kms,gretton2012kernel}. The geometry limits the task mixtures considered by discovery. Coverage comes from the frozen rule cluster inference, which follows finite dimensional multiplier bootstrap arguments \citep{chernozhukov2013gaussian}. Industry guidance also treats agent evaluation as a release and governance process. Anthropic discusses the difficulty of evaluating agents that use tools and change state, Microsoft provides repeated evaluation for business agents, and NIST is developing agent security evaluations and standards \citep{anthropic2026agentevals,nist2026agentstandards}.

\section{Proofs and technical details}
\label{app:proofs}

\subsection{Normalized weighted-ratio expansion}

\begin{lemma}[Weighted-ratio expansion]
\label{lem:ratio}
Let $a_g=P s_g>0$, $P^g h=P(s_gh)/a_g$, and $P_n^g h=P_n(s_gh)/P_ns_g$. Whenever $P_ns_g>0$,
\begin{equation}
 P_n^g h-P^g h
 =\frac{a_g}{P_ns_g}(P_n-P)
 \left\{\omega_g(h-P^gh)\right\},
 \qquad \omega_g=s_g/a_g.
 \label{eq:ratioexact}
\end{equation}
If $P\{\omega_g^2(h-P^gh)^2\}<\infty$ and $P_ns_g\to_p a_g$, then
\begin{equation}
 \sqrt n(P_n^g h-P^g h)
 =\frac{1}{\sqrt n}\sum_{i=1}^n
 \omega_g(W_i)\{h(W_i)-P^gh\}+o_p(1).
 \label{eq:ratiolin}
\end{equation}
The same statement holds coordinatewise for a fixed-dimensional vector $h$.
\end{lemma}

\begin{proof}
Put the empirical and population tilted means over the common denominator $P_ns_g$:
\[
 P_n^gh-P^gh
 =\frac{P_n\{s_g(h-P^gh)\}}{P_ns_g}
 =\frac{a_g}{P_ns_g}P_n\{\omega_g(h-P^gh)\}.
\]
Because $P\{\omega_g(h-P^gh)\}=0$, this is exactly~\eqref{eq:ratioexact}. The centered empirical average is $O_p(n^{-1/2})$ under the stated second-moment condition, while $a_g/(P_ns_g)=1+o_p(1)$. Slutsky's theorem gives~\eqref{eq:ratiolin}. The vector statement follows coordinatewise because the dimension is fixed.
\end{proof}

\subsection{Cluster linearization and plug-in contributions}

\begin{lemma}[Weighted-report cluster linearization]
\label{lem:clusterlin}
Conditional on $\mathcal F_{\Dcal}$, suppose the conditions of Theorem~\ref{thm:validity} hold. Let $\mu_m=\E(H_{1,m}\mid\mathcal F_{\Dcal})$ and $\Phi_{k,m}=\nabla F_m(\mu_m)^\top(H_{k,m}-\mu_m)$. Uniformly over the finite family $\Jcal_0$,
\begin{equation}
 \sqrt N(\widehat\Delta_m-\Delta_m)
 =N^{-1/2}\sum_{k=1}^N\Phi_{k,m}+o_p(1),
 \label{eq:clusterlin}
\end{equation}
and the canonical plug-in contributions satisfy
\begin{equation}
 \max_{m\in\Jcal_0}N^{-1}\sum_{k=1}^N
 (\widehat\Phi_{k,m}-\Phi_{k,m})^2\xrightarrow{p}0.
 \label{eq:plugincons}
\end{equation}
\end{lemma}

\begin{proof}
Every population weighted scoring and held out denominator is at least a fixed $b_0>0$. The cluster law of large numbers and finiteness of $\Jcal_0$ imply that, with probability tending to one, all empirical denominators are at least $b_0/2$. Hence every $F_m$ is evaluated inside a neighborhood where its gradient is bounded and uniformly continuous.

Finite second moments and the cluster Lindeberg condition give $\widehat\mu_m-\mu_m=O_p(N^{-1/2})$. Taylor's theorem with integral remainder yields
\begin{align*}
 F_m(\widehat\mu_m)-F_m(\mu_m)
 &=\nabla F_m(\mu_m)^\top(\widehat\mu_m-\mu_m)+R_{N,m},\\
 |R_{N,m}|
 &\le \|\widehat\mu_m-\mu_m\|
 \sup_{0\le t\le1}
 \|\nabla F_m(\mu_m+t(\widehat\mu_m-\mu_m))-
 \nabla F_m(\mu_m)\|.
\end{align*}
The first factor is $O_p(N^{-1/2})$ and the second is $o_p(1)$, uniformly over the finite family. The Taylor expansion proves~\eqref{eq:clusterlin}. The coordinates of $H_{k,m}$ include the weighted scoring and held out numerators and denominators, along with the selector and aggregation inputs. Differentiating $q(S)^\top T$ gives
\[
 d\{q(S)^\top T\}=q(S)^\top dT+\{Dq(S)dS\}^\top T,
\]
so $\Phi_{k,m}$ contains both held out evaluation and splitwise selection effects. Held out evaluation and splitwise selection are the two channels through which an evaluation unit changes an adaptive report.

For the plug-in result, write
\begin{align*}
 \widehat\Phi_{k,m}-\Phi_{k,m}
 &=\{\nabla F_m(\widehat\mu_m)-\nabla F_m(\mu_m)\}^\top(H_{k,m}-\mu_m)\\
 &\quad-\nabla F_m(\widehat\mu_m)^\top(\widehat\mu_m-\mu_m).
\end{align*}
The average squared first term is $o_p(1)$ by gradient continuity, Cauchy Schwarz, and bounded empirical second moments. The second term is common across $k$ and is $o_p(1)$ because the gradient is bounded and $\widehat\mu_m-\mu_m=o_p(1)$. Taking the maximum over the finite family proves~\eqref{eq:plugincons}.
\end{proof}

\subsection{Proof of Theorem~\ref{thm:validity}}

\begin{proof}
Condition throughout on $\mathcal F_{\Dcal}$. The discovered rules and their random targets are then fixed. Let
\[
 \Phi_k=(\Phi_{k,m}:m\in\Jcal_0)^\top,
 \qquad
 S_N=N^{-1/2}\sum_{k=1}^N\Phi_k.
\]
By Lemma~\ref{lem:clusterlin}, the centered estimator vector equals $S_N+o_p(1)$. For every fixed $t\in\mathbb R^{|\Jcal_0|}$, the scalar triangular array $N^{-1/2}\sum_k t^\top\Phi_k$ satisfies the Lindeberg Feller conditions and has variance converging to $t^\top\Sigma t$. The scalar CLT and Cram\'er--Wold device therefore give
\[
 S_N\xrightarrow{d}N_{|\Jcal_0|}(0,\Sigma).
\]
Slutsky's theorem proves the asserted joint limit for the estimator.

It remains to establish the bootstrap approximation. Let
\[
 \widehat\Sigma
 =N^{-1}\sum_{k=1}^N
 (\widehat\Phi_k-\overline{\widehat\Phi})
 (\widehat\Phi_k-\overline{\widehat\Phi})^\top.
\]
Let $D_k=\widehat\Phi_k-\Phi_k$. Lemma~\ref{lem:clusterlin} gives $N^{-1}\sum_k\|D_k\|^2=o_p(1)$. Cauchy Schwarz then shows that every cross term $N^{-1}\sum_kD_k\Phi_k^\top$ is $o_p(1)$, and the $D_kD_k^\top$ terms are also $o_p(1)$. The cluster law of large numbers and covariance assumption imply that the sample covariance of the true $\Phi_k$ converges to $\Sigma$. Hence
\[
 \widehat\Sigma\xrightarrow{p}\Sigma.
\]

Conditional on the inference data, the multiplier vector
\[
 G^*=N^{-1/2}\sum_{k=1}^N
 \zeta_k(\widehat\Phi_k-\overline{\widehat\Phi}),
 \qquad \zeta_k\stackrel{\mathrm{iid}}{\sim}N(0,1),
\]
is exactly Gaussian with mean zero and covariance $\widehat\Sigma$. Continuity of centered Gaussian laws in their covariance matrices therefore yields conditional weak convergence to $N(0,\Sigma)$ in bounded Lipschitz distance. Coordinatewise bootstrap quantiles give pointwise coverage.

For signed simultaneous claims, fix $\eta_m\in\{-1,+1\}$ and let $\sigma_m^2=\Sigma_{mm}$. On the reported nondegenerate coordinates assume $\min_m\sigma_m^2>0$. The map
\[
 x\longmapsto \max_{m\in\Jcal_0}\frac{\eta_m x_m}{\sigma_m}
\]
is continuous, and its finite dimensional nondegenerate Gaussian distribution is continuous. Conditional bootstrap quantiles of the corresponding studentized maximum are therefore consistent. The resulting lower bands $B_m$ for $\eta_m\Delta_m$ satisfy
\[
 \Pp\!\left(
 B_m\le \eta_m\Delta_m\ \text{for all }m\in\Jcal_0
 \mid\mathcal F_{\Dcal}
 \right)
 \ge 1-\alpha-o_p(1).
\]
On this event the inequality holds for every coordinate, so it also holds for any coordinate selected after inspecting the estimates and bands. The simultaneous event proves the pointwise, simultaneous, and selected report conclusions.
\end{proof}

\section{Expanded operational algorithm}
\label{app:algorithm}

The following details make Algorithm~\ref{alg:sasst} fully executable. The algorithm is common to static LLM evaluations and agentic workflows; only the evaluation-unit representation, cluster definition, score tensor, and allowed feature map change.

\subsection{Frozen inputs and reporting slots}

The inputs are: independent clusters $C_1,\ldots,C_N$; the completed reporting layer $\Acal$; a score blind feature map $\phi$; and a finite family $\Jcal_0$. A slot $m\in\Jcal_0$ fixes the contrast $c_m$, one sided sign $\eta_m\in\{-1,+1\}$, feature/kernel protocol, candidate family $\mathcal G_m$, KMS radius $\rho_m$, density cap $\kappa_m$, minimum effective sample size $n_{\min,m}$, denominator thresholds, stability threshold, optimization tolerance, and deterministic tie rule. A candidate must be a reusable object $g$ that maps any allowed feature vector to a nonnegative score $s_g(W)$; a transductive vector of discovery weights is not an admissible output.

Split whole clusters into $\Dcal$ and $\Ical$. All encoders, standardization statistics, random Fourier features, KMS witness directions, dictionaries, split schedules, and baseline systems are frozen according to the slot before confirmatory outcomes are read.

\subsection{Discovery on $\Dcal$}

For each contrast needed by a slot, run the unweighted fixed reporting layer on $\Dcal$ and obtain $\widehat\Delta_{\Dcal,c}$ and first-order contributions $\widehat\psi_{c,i}$. For every $g\in\mathcal G_m$, compute
\[
 w_i(g;\Dcal)=\frac{s_g(W_i)}{\sum_{\ell\in\Dcal}s_g(W_\ell)},\qquad
 \widehat\ESS_g=\left\{\sum_{i\in\Dcal}w_i(g;\Dcal)^2\right\}^{-1},\qquad
 \widehat\kappa_g=|\Dcal|\max_i w_i(g;\Dcal).
\]
Let $\widehat P_{\Dcal}$ be uniform on the discovery features and let $\widehat P_{\Dcal}^{g}$ use the weights above. With a frozen nonlinear feature map $z$ and finite witness set $\mathcal B_L$, the implemented radius is
\[
 \widehat{\KMS}_{1,L}(\widehat P_{\Dcal},\widehat P_{\Dcal}^{g})
 =\max_{\beta\in\mathcal B_L}
 \Wass_1\!\left((\beta^\top z)_{\#}\widehat P_{\Dcal},
                 (\beta^\top z)_{\#}\widehat P_{\Dcal}^{g}\right).
\]
A candidate is discovery-feasible when its density, $\ESS$, KMS, source/template, and other predeclared score blind constraints pass. Among feasible candidates, select
\[
 \widehat g_m\in\arg\min_{g\in\mathcal G_m}
 \left[
 \widehat\Delta_{\Dcal,c_m}
 +\sum_{i\in\Dcal}\left\{w_i(g;\Dcal)-|\Dcal|^{-1}\right\}
 \widehat\psi_{c_m,i}
 \right],
\]
up to the declared optimization tolerance and fixed tie rule. If no candidate is feasible, record failed discovery. Serialize every fitted threshold, basis, center, bandwidth, preprocessing statistic, and diagnostic needed to apply $\widehat g_m$ to new units.

Discovery stability is evaluated by resampling discovery clusters, rerunning the same slot, and comparing the resulting weight functions on a frozen reference feature set. Discrete rules may use exact agreement; continuous rules use a predeclared functional metric such as weight correlation and top-weighted-set overlap. Stability never uses inference outcomes.

\subsection{Frozen-rule confirmation on $\Ical$}

Apply $s_{\widehat g_m}$ to inference features. Recompute the score-blind density ratio, $\ESS$, source/template caps, and every scoring and held-out denominator. If any predeclared inference feasibility check fails, record failed feasibility and stop that slot. If support passes but the discovery description fails its predeclared stability criterion, set an abstention flag and continue so that the frozen-rule estimate and bound can be reported diagnostically.

For every support-feasible slot, let $S_r$ and $E_r$ be the fixed scoring and held-out subsets in repeated split $r$. For system $j$ and retained artifact $a$, compute the exact ratios
\[
 \widehat S^{g}_{r,j,a}
 =\frac{\sum_{i\in S_r}s_g(W_i)Z_j(W_i,a)}{\sum_{i\in S_r}s_g(W_i)},
 \qquad
 \widehat T^{g}_{r,j,a}
 =\frac{\sum_{i\in E_r}s_g(W_i)Z_j(W_i,a)}{\sum_{i\in E_r}s_g(W_i)}.
\]
Apply the frozen selector and aggregation without modification:
\[
 \widehat q^{g}_{r,j}=\pi_j(\widehat S^{g}_{r,j}),\qquad
 \widehat Y^{g}_{r,j}=(\widehat q^{g}_{r,j})^\top\widehat T^{g}_{r,j},\qquad
 \widehat\theta^{g}_{j}=\sum_r\lambda_r\widehat Y^{g}_{r,j},\qquad
 \widehat\Delta_m=c_m^\top\widehat\theta^{g}.
\]
Both sides of a paired comparison are included in this same contrast. Construct the cluster tensor $H_{k,m}$ from all weighted numerators and denominators, selector inputs, aggregation coordinates, and systems in the contrast, and compute
\[
 \widehat\Phi_{k,m}=\nabla F_m(\widehat\mu_m)^\top(H_{k,m}-\widehat\mu_m).
\]
The exact derivative contains both the held out evaluation term and the splitwise selection term displayed in Lemma~\ref{lem:clusterlin}.

\subsection{Inference rule and small cluster safeguard}

The Gaussian multiplier maximum is the default asymptotic procedure. The inference rule must be frozen before confirmation outcomes are read. When a score-blind pilot or an independent simulation places the study in a small-cluster regime that has not passed the multiplier calibration check, the protocol may instead use an independently audited conservative family. Our audited safeguard is the one-sided Bonferroni $t$ band
\[
 B_m^{\mathrm{Bonf}}
 =\eta_m\widehat\Delta_m
 -t_{1-\alpha/|\Jcal_0|,\,N-1}
 \frac{\widehat\sigma_m}{\sqrt N}.
\]
Theorem~\ref{thm:validity} does not imply this finite-sample safeguard. It is an operating rule supported by the small-cluster simulations in Appendix~\ref{app:experiments}. The Gaussian analysis was the original Study F analysis; Bonferroni $t$ was a later conservative sensitivity, and both give the same status for every Study F slot. Future small-cluster studies should freeze the safeguarded rule before confirmation outcomes are read.

\subsection{Simultaneous reporting and statuses}

Use one multiplier $\zeta_k$ per inference cluster and share it across all slots. For predeclared signs $\eta_m$, apply the multiplier maximum to the signed contributions $\eta_m\widehat\Phi_{k,m}$. A fragility claim uses $\eta_m=-1$, so a lower band for $-\Delta_m$ is an upper bound for $\Delta_m$. Report the complete attempted family, including nonreported slots:
\begin{center}
\footnotesize
\begin{tabular}{p{0.22\linewidth}p{0.70\linewidth}}
\toprule
Status & Meaning \\
\midrule
confirmed & Feasibility and stability pass, and the predeclared signed band supports the stated claim. \\
not confirmed & Feasibility and stability pass, but the signed band does not support the predeclared effect threshold. \\
failed discovery & No candidate in the slot satisfies the discovery constraints. \\
failed feasibility & The frozen rule lacks inference support, violates a cap, or has an unstable scoring or held out denominator. \\
abstained & The frozen rule effect may be estimable, but its semantic or workflow description fails the declared replicability or auditability criterion. \\
\bottomrule
\end{tabular}
\end{center}
The output row contains the frozen rule and description, ordinary and stressed contrasts, signed bound, $\ESS$, maximum density ratio, minimum split denominator, stability, geometry diagnostics, runtime, and audit/control recommendation. Any external retrieval, generation, human curation, or mitigation experiment is a separate generalization study and is not part of the benchmark relative coverage certificate.

\section{Corrected experiments and audits}
\label{app:experiments}

\paragraph{Artifact policy.}
Only corrected and versioned artifacts from the independent code and artifact audit support the results below. Superseded first pass analyses are excluded from every paper table and claim. The original full paper plan reserved Study E for a larger static evaluation. The workshop paper does not claim a completed Study E.

\subsection{Study A and A4: calibration after discovery}

Study A uses eight Gaussian score blind features, five artifacts with different stress responses, a fixed paired baseline, nineteen bounded candidate rules, and a 50/50 discovery and confirmation split. Each configuration uses 2,000 trials and 2,000 multiplier draws.

\begin{table}[H]
\centering
\scriptsize
\caption{Study A pointwise coverage in percent. Fixed means a fixed oracle rule, and learned means a rule chosen on discovery data.}
\label{tab:app_study_a}
\begin{tabular}{rcccc}
\toprule
$M$ & Fixed, $R=5$ & Learned, $R=5$ & Fixed, $R=20$ & Learned, $R=20$ \\
\midrule
500  & 95.00 & 93.85 & 94.60 & 94.05 \\
1000 & 94.75 & 95.10 & 95.10 & 94.75 \\
2000 & 95.10 & 95.30 & 94.95 & 95.05 \\
4000 & 95.30 & 95.15 & 95.75 & 95.10 \\
\bottomrule
\end{tabular}
\end{table}

\begin{table}[H]
\centering
\scriptsize
\caption{Matched same data comparison and Study A4 signed family results.}
\label{tab:app_study_a4}
\begin{tabular}{lccc}
\toprule
Setting & Split procedure & Same data procedure & Protected family result \\
\midrule
$M=500$ & 94.4\% coverage & 91.3\% coverage & not run \\
$M=1000$ & 93.6\% coverage & 91.6\% coverage & 94.95\% familywise, 0.85\% false fragility \\
$M=2000$ & 94.4\% coverage & 93.3\% coverage & 95.50\% familywise, 0.55\% false fragility \\
\bottomrule
\end{tabular}
\end{table}

At $M=1000$, the paired split minus same data coverage gap is 2.0 percentage points, with one sided lower bound 0.31 percentage points and McNemar $p=0.032$. Fixed rule coverage ranges from 94.60\% to 95.75\%. For $M\ge1000$, learned rule coverage ranges from 94.75\% to 95.30\%, with negligible bias and estimated to empirical standard error ratios from 0.991 to 1.018.

\subsection{Small cluster audit and safeguard}

Study F has only $N=40$ inference clusters, placing it in a regime where the Gaussian multiplier bootstrap may undercover. This subsection quantifies the gap and identifies a conservative alternative. The corrected small sample audit uses the exact Study F estimator and the complete ten slot family.

\begin{table}[H]
\centering
\scriptsize
\caption{Small cluster familywise coverage across four audited null settings (FSV-remedy, post-hoc fresh-seed diagnostic). Values are ranges over the settings.}
\label{tab:app_fsv}
\begin{tabular}{lcc}
\toprule
Inference rule & Familywise coverage & Reading \\
\midrule
Gaussian multiplier maximum & 90.55 to 93.50\% & below the declared threshold \\
HC1 variance inflation & 91.95 to 93.85\% & not adequate \\
Common $t$ scaling & 93.75 to 95.20\% & not uniformly adequate \\
Bonferroni one sided $t$ & 96.40 to 97.55\% & conservative in all four settings \\
\bottomrule
\end{tabular}
\end{table}

\begin{table}[H]
\centering
\scriptsize
\caption{Power for a negative ten point stressed effect under density-feasible designs (FSV-remedy, fresh seeds).}
\label{tab:app_power}
\begin{tabular}{rrrrr}
\toprule
Clusters & Stressed $\ESS$ & Gaussian power & Bonferroni $t$ power & Reading \\
\midrule
40  & 20  & 10.95\% & 3.90\% & low support \\
80  & 20  & 9.70\% & 4.35\% & more clusters do not help at fixed $\ESS$ \\
200 & 40  & 18.90\% & 14.45\% & moderate support \\
200 & 100 & 48.10\% & 40.25\% & support improves power \\
\bottomrule
\end{tabular}
\end{table}

Gaussian familywise coverage falls 1.5--4.5 points below the 95\% target at $N=40$. The nominal 200-cluster/ESS-20 design is excluded because its maximum density ratio is 10.26, exceeding $\kappa=5$. The theorem is asymptotic, while the Bonferroni rule is a tested operating safeguard rather than a theorem consequence.

\subsection{Corrected positive control}

The small cluster audit shows the method is conservative; this subsection shows it is not vacuous by demonstrating that SASST can confirm a real fragility when one exists at sufficient scale. The positive control is a same-seed diagnostic (not fresh validation) using frozen $\tau$-bench features, the Study F estimator, independent confirmation clusters, a planted fragility in the claim direction, and Bonferroni $t$ inference. The conditional confirmation rate is 0.701 with MC95 interval $[0.671, 0.730]$, barely clearing its 0.700 gate.

\begin{table}[H]
\centering
\scriptsize
\caption{Post-audit same-seed positive-control diagnostic.}
\label{tab:app_fpos}
\begin{tabular}{lccc}
\toprule
Gate & Value & Required value & Result \\
\midrule
Null familywise error, $N=40$ & 0.017 & at most 0.075 & pass \\
Null familywise error, $N=200$ & 0.018 & at most 0.075 & pass \\
Target rule discovery, $N=200$ & 0.921 & at least 0.600 & pass \\
Conditional confirmation, $N=200$ & 0.701 & at least 0.700 & pass \\
End to end detection, $N=200$ & 0.646 & at least 0.500 & pass \\
\bottomrule
\end{tabular}
\end{table}

\subsection{Study B and the geometry audit}

Study B compares KMS, MMD, linear max-sliced, no radius, and a restricted linear registry. A corrected Wasserstein routine changes 2.125\% of KMS selections and 0.625\% of max-sliced selections, but it does not change the qualitative result.

\begin{table}[H]
\centering
\scriptsize
\caption{Corrected Study B held-out AUC after KMS replay (mean over 200 trials). Bold marks the best method per mechanism.}
\label{tab:app_study_b}
\begin{tabular}{lccccc}
\toprule
Mechanism & KMS & MMD & Max-sliced & No-radius & Linear-only \\
\midrule
Linear & 0.602 & \textbf{0.631} & 0.623 & 0.597 & 0.594 \\
XOR & 0.639 & \textbf{0.747} & 0.610 & 0.649 & 0.631 \\
Circular & 0.695 & \textbf{0.716} & 0.696 & 0.558 & 0.500 \\
Policy by tool & 0.698 & 0.679 & \textbf{0.707} & 0.700 & 0.576 \\
\bottomrule
\end{tabular}
\end{table}

MMD is strongest on three of four mechanisms; max-sliced wins on policy by tool. No geometry wins in every mechanism. Study B supports nonlinear rule recovery for some mechanisms; it does not show that KMS is universally best or provide a confirmed natural fragility result.

A separate geometry audit uses frozen $\tau$-bench features and a planted short goal mechanism.

\begin{table}[H]
\centering
\scriptsize
\caption{Geometry audit under one frozen absolute radius budget (fixed registry axis, mean over six kernel seeds).}
\label{tab:app_fgeo}
\begin{tabular}{lccc}
\toprule
Geometry & Accepted rules & Correct rule recovery & AUC \\
\midrule
Standardized features & 53.0 & 0.000 & 0.579 \\
Normalized plus structured features & 50.7 & 0.000 & 0.526 \\
PCA plus structured features & 50.7 & 0.000 & 0.522 \\
Structured features only & 52.5 & 0.000 & 0.527 \\
All geometries without radius constraint & 71 & 0.610 & 0.814 \\
\bottomrule
\end{tabular}
\end{table}

The target rule is outside the common 75\% radius budget for every geometry. Median unit bandwidth RBF similarity in the standardized high dimensional space is about $4\times10^{-177}$. The radius and bandwidth must therefore be checked on score blind pilot data.

\subsection{Study C: feasibility and abstention}

Study C verifies that the support gates (density, $\ESS$, split denominators, stability) correctly block infeasible or unstable rules across five synthetic settings.

\begin{table}[H]
\centering
\scriptsize
\caption{Study C operating checks over 500 trials per setting.}
\label{tab:app_study_c}
\begin{tabular}{lccc}
\toprule
Setting & Abstention or failure & Nonabstention & Recorded check \\
\midrule
Severe concentration & 100.0\% & 0.0\% & $\ESS$ and density \\
Density only & 100.0\% & 0.0\% & density, while $\ESS$ passes \\
Split support only & 100.0\% & 0.0\% & split support, while global checks pass \\
Near tied rules & 95.8\% & 4.2\% & stability \\
Stable positive rule & 0.0\% & 100.0\% & none \\
\bottomrule
\end{tabular}
\end{table}

The 4.2\% near tie nonabstention rate is not a false fragility rate. Study A4 measures false fragility separately.

\subsection{Study D: partly synthetic real feature study}

Study D uses 1,001 MMLU-Pro Law items, Qwen3-8B score tensors, 384 dimensional text features, and an injected degradation that flips 70\% of correct answers within a 40-item jurisprudence subgroup.

\begin{table}[H]
\centering
\scriptsize
\caption{Study D result and explanation diagnostics.}
\label{tab:app_study_d}
\begin{tabular}{lc}
\toprule
Quantity & Value \\
\midrule
Weighted confirmation contrast & $-0.0554$ \\
Pointwise 95\% interval & $[-0.0886,-0.0221]$ \\
Unweighted confirmation contrast & $-0.0294$ \\
Stress amplification & $-0.0260$ \\
Recovery AUC & 0.974 \\
Top weighted precision and Jaccard & 0.444 \\
Top weighted recall & 1.000 \\
Inference $\ESS$ & 405.7 \\
Functional stability, 100 resamples & 0.18 \\
Formal status & abstained \\
\bottomrule
\end{tabular}
\end{table}

The corrected radius is 0.003052. The frozen-rule effect is detectable, but functional stability is 0.18, so the formal status remains abstained and the description is not reusable. The high AUC is a ranking result and does not imply exact subgroup recovery.

\subsection{Controller mechanical audit}

Before interpreting Study F and G workflow differences, the routing code must be verified: does ReAct pass everything through, does the verifier block on rejection, and does the confirmation gate respect consent tokens? A fixed deck sends sixteen state changing calls to the three workflows. Oracle verifier signals isolate controller mechanics from model judgment quality.

\begin{table}[H]
\centering
\scriptsize
\caption{Controller replay. The test measures mechanics rather than semantic verifier quality.}
\label{tab:app_fctl}
\begin{tabular}{lcccc}
\toprule
Workflow & Invalid block recall & Valid approval & Confirmed executed & Declined blocked \\
\midrule
ReAct pass through & 0.0\% & 100.0\% & 100.0\% & 0.0\% \\
Planner with oracle verifier & 100.0\% & 100.0\% & 50.0\% & 50.0\% \\
Confirmation gate & 50.0\% & 50.0\% & 100.0\% & 100.0\% \\
\bottomrule
\end{tabular}
\end{table}

All twenty declared mechanical gates and ten boundary and regression tests pass. The prompt contains the exact serialized mutation and full arguments, and malformed verifier output fails closed. This audit applies to the revised \texttt{f2\_v2} controller; Studies F and G used the historical \texttt{study\_f\_v1} controller, whose routing mechanics are consistent but whose confirmation prompt disclosed only the tool name. The audit does not validate informed consent or model based semantic verification.

\subsection{Studies F and G: interactive agent evaluations}

Study F uses Qwen3-8B. Study G uses Qwen2.5-7B-Instruct. Both use the same eighty task IDs, two simulator seeds, three workflows, 40/40 task cluster split, frozen features, three radii, and ten reporting slots. Each study contains 480 episodes.

\begin{table}[H]
\centering
\scriptsize
\caption{Across-model agent summary. Study F full-study rates combine the exact 80-observation discovery and confirmation halves; Study G rates are reported over its full 160 observations per workflow.}
\label{tab:agent_appendix}
\begin{tabular}{lcc}
\toprule
Quantity & Study F, Qwen3-8B & Study G, Qwen2.5-7B-Instruct \\
\midrule
Full-study ReAct success & 18.75\% & 11.9\% \\
Full-study Plan + verifier success & 20.625\% & 11.2\% \\
Full-study confirmation-gate success & 19.375\% & 14.4\% \\
Discovery choice & Plan + verifier & Confirmation gate \\
Plan + verifier minus ReAct on confirmation & 0.000 & +0.0125 \\
Confirmation gate minus ReAct on confirmation & $-0.0250$ & +0.0250 \\
Maximum rule stability & 0.37 & 0.27 \\
Confirmed & 0 & 0 \\
Not-confirmed & 2 & 2 \\
Abstained & 8 & 6 \\
Failed-feasibility & 0 & 2 \\
\bottomrule
\end{tabular}
\end{table}

\begin{table}[H]
\centering
\scriptsize
\caption{Study F simultaneous ten-slot report under the original Gaussian procedure (critical value 2.4632). Bonferroni $t$ (critical 2.7424) widens bands by factor 1.113 and leaves every status unchanged.}
\label{tab:app_study_f_slots}
\resizebox{\textwidth}{!}{%
\begin{tabular}{lrrrrrrl}
\toprule
Slot & Estimate & SE & Simultaneous bound & $\ESS$ & Density & Stability & Status \\
\midrule
Uniform: Plan + verifier $-$ ReAct & 0.0000 & 0.0791 & lower $-0.195$ & 40.0 & 1.00 & 1.00 & not confirmed \\
Uniform: confirmation gate $-$ ReAct & $-0.0250$ & 0.0498 & lower $-0.148$ & 40.0 & 1.00 & 1.00 & not confirmed \\
$r_1=0.01017$: Plan + verifier & $-0.0227$ & 0.0782 & upper 0.170 & 34.6 & 2.73 & 0.05 & abstained \\
$r_1=0.01017$: confirmation gate & $-0.0227$ & 0.0454 & upper 0.089 & 34.6 & 2.73 & 0.05 & abstained \\
$r_2=0.01799$: Plan + verifier & +0.0400 & 0.0836 & upper 0.246 & 31.3 & 2.40 & 0.07 & abstained \\
$r_2=0.01799$: confirmation gate & 0.0000 & 0.0490 & upper 0.121 & 31.3 & 2.40 & 0.07 & abstained \\
$r_3=0.04003$: Plan + verifier & +0.1298 & 0.1398 & upper 0.474 & 20.6 & 2.40 & 0.37 & abstained \\
$r_3=0.04003$: confirmation gate & +0.0178 & 0.0801 & upper 0.215 & 20.6 & 2.40 & 0.37 & abstained \\
Control $r_2$: Plan $-$ gate & +0.0400 & 0.0574 & lower $-0.101$ & 31.3 & 2.40 & 0.07 & abstained \\
Control $r_2$: gate $-$ Plan & $-0.0400$ & 0.0574 & lower $-0.181$ & 31.3 & 2.40 & 0.07 & abstained \\
\bottomrule
\end{tabular}}
\end{table}

Study F's maximum density ratio is 2.73, well within $\kappa=5$. Study G's third radius fails feasibility because $\ESS=18.7$ falls below the predeclared minimum of 20.

\begin{table}[H]
\centering
\scriptsize
\caption{Study G ten-slot report under the preregistered Gaussian simultaneous procedure (critical value 2.347). The Gaussian band is not well calibrated at $N=40$ (Section~\ref{app:experiments}); Bonferroni $t$ widens all intervals without changing any status. NA means unavailable after feasibility failure.}
\label{tab:app_study_g_slots}
\resizebox{\textwidth}{!}{%
\begin{tabular}{lrrrrrrl}
\toprule
Slot & Estimate & SE & Simultaneous bound & $\ESS$ & Density & Stability & Status \\
\midrule
Uniform: Plan + verifier $-$ ReAct & +0.0125 & 0.0414 & lower $-0.085$ & 40.0 & 1.00 & 1.00 & not confirmed \\
Uniform: confirmation gate $-$ ReAct & +0.0250 & 0.0393 & lower $-0.067$ & 40.0 & 1.00 & 1.00 & not confirmed \\
Stress $r_1$: Plan + verifier & +0.0125 & 0.0414 & upper 0.110 & 40.0 & 1.00 & 0.10 & abstained \\
Stress $r_1$: confirmation gate & +0.0250 & 0.0393 & upper 0.117 & 40.0 & 1.00 & 0.10 & abstained \\
Stress $r_2$: Plan + verifier & +0.0145 & 0.0435 & upper 0.117 & 34.8 & 1.16 & 0.27 & abstained \\
Stress $r_2$: confirmation gate & +0.0021 & 0.0411 & upper 0.099 & 34.8 & 1.16 & 0.27 & abstained \\
Stress $r_3$: Plan + verifier & +0.0356 & NA & NA & 18.7 & 2.85 & 0.14 & failed feasibility \\
Stress $r_3$: confirmation gate & +0.0711 & NA & NA & 18.7 & 2.85 & 0.14 & failed feasibility \\
Control $r_2$: Plan $-$ gate & +0.0124 & 0.0384 & lower $-0.078$ & 34.8 & 1.16 & 0.27 & abstained \\
Control $r_2$: gate $-$ Plan & $-0.0124$ & 0.0384 & lower $-0.103$ & 34.8 & 1.16 & 0.27 & abstained \\
\bottomrule
\end{tabular}}
\end{table}

Study G uses the same tasks and is therefore an across model replication of the reporting decision rather than an independent benchmark replication. It selects a different workflow and different stress rules, but it also produces zero confirmed claims.

\begin{table}[H]
\centering
\scriptsize
\caption{Study F descriptive execution counts. The counts are not causal safety effects.}
\label{tab:app_resources}
\begin{tabular}{lrrrr}
\toprule
Workflow & Tokens & Tool calls & Failed calls & Irreversible calls \\
\midrule
ReAct & 19,932,313 & 1,091 & 217 & 269 \\
Plan + verifier & 17,535,902 & 1,058 & 157 & 176 \\
Confirmation gate & 19,728,565 & 1,021 & 200 & 233 \\
\bottomrule
\end{tabular}
\end{table}

Study F used Qwen3-8B for 480 episodes, totaling 57,196,780 model tokens and 18,280 summed episode seconds. Study G used Qwen2.5-7B-Instruct for 480 episodes and required about 2.5 hours of collection at three concurrent workers plus 30 seconds of inference. The recorded execution environment used an NVIDIA RTX 6000 Ada Generation (49,140\,MiB); a separate Study G environment manifest was not preserved.

Study F sensitivity analyses answer different questions. Treating three parser errors as missing leaves the two confirmation contrasts and all statuses unchanged. A stronger balanced deletion lowers stressed $\ESS$ for two slots below the threshold, changing those slots from abstained to failed feasibility while keeping zero confirmed claims.

\paragraph{Post-hoc descriptive decomposition.}
Across all task-seed blocks, 104/160 in Study F and 118/160 in Study G scored zero under all three workflows. These are observed all-workflow-zero blocks, not estimates of a structural floor. In Study G, Plan + verifier reached environment completion in 76/160 episodes versus 118/160 for ReAct. Trace termination codes show that the 84 Plan + verifier non-completions comprise 32 max-step terminations and 52 verifier-rejection terminations, so the completion difference primarily reflects controller-induced stopping. Conditional success among completed episodes was 23.7\% for Plan + verifier and 16.1\% for ReAct. Because completion is affected by workflow assignment, these completed sets are selected post-treatment populations; the difference should not be interpreted as a causal quality improvement.

\subsection{Reproducibility and superseded analyses}

The paper tables use only corrected artifacts. The arithmetic audit verifies the discovery gain, confirmation contrasts, protected bounds, critical values, and status counts. The claim source matrix maps each paper claim to a corrected artifact and records allowed and forbidden wording. An exclusion audit checks that no paper script reads a superseded first pass result.

\begin{table}[H]
\centering
\scriptsize
\caption{Artifact provenance summary.}
\label{tab:app_provenance}
\begin{tabular}{llc}
\toprule
Analysis & Provenance & Seeds \\
\midrule
Study A, A4, B, C, D & Preregistered simulation & dedicated \\
Study F, G & Preregistered interactive & dedicated \\
KMS correction & Corrective overlay (frozen Study F) & same \\
FSV-v2 & Independent calibration audit & fresh \\
FSV-remedy & Post-hoc fresh-seed diagnostic & fresh \\
FGEO-v2 & Corrective geometry audit & fresh \\
FPOS-v2 & Same-seed diagnostic & same \\
FCTL2 & Mechanical audit (f2\_v2 controller) & deterministic \\
\bottomrule
\end{tabular}
\end{table}

A prospective natural follow up would require fresh tasks, a valid outcome aligned with the policy claim, a predeclared small sample inference rule, corrected feature geometry, a nonredundant intervention, and a power plan based on stressed effective sample size. We leave that experiment to later work.

\section{Limitations and broader impact}
\label{app:limitations}

The theorem is asymptotic and assumes bounded density ratios, stable denominators, honest discovery, and plausibly independent clusters. The forty-cluster audit shows that the generic Gaussian family can undercover in a small sample, while the independently audited Bonferroni $t$ safeguard is conservative in the tested settings. The safeguard has been audited in four null settings rather than proved for every finite sample.

The stress rule search is limited by the declared registry, feature map, radius, and support thresholds. A rule outside the radius is outside the audit scope. High-dimensional kernel features can also become nearly orthogonal when the bandwidth is poorly scaled. The agent studies use one benchmark task pool, low-success open models, the same model as actor and user simulator, and forty confirmation task clusters. The second-model study is a model replication on the same tasks rather than an independent benchmark replication.

The intended use is evaluation and release control. A team can use the protocol to distinguish a supported workflow claim from a development-set gain, an unstable segment, or a low-support analysis. Misuse would include treating a benchmark relative result as proof of deployment safety, treating non-confirmation as equivalence, or selecting a new analysis after seeing confirmation outcomes. The paper rejects those interpretations.

\end{document}